\documentclass[11pt]{article}

\usepackage[letterpaper,margin=1in]{geometry}

\usepackage{amsmath}
\usepackage{amssymb}
\usepackage{amsthm}

\usepackage{booktabs}
\usepackage{array}

\usepackage{tikz}
\usepackage{pgfplots}
\pgfplotsset{compat=1.17}

\usepackage[hidelinks]{hyperref}

\newtheorem{definition}{Definition}
\newtheorem{proposition}{Proposition}
\newtheorem{observation}{Observation}

\title{Descriptive Collision in Sparse Autoencoder Auto-Interpretability:\\
{\large When One Explanation Describes Many Features}}
\author{Jordan F.\ McCann\\
Independent Researcher\\
\texttt{jordanfmccann@gmail.com}}
\date{May 11, 2026}

\begin{document}
\maketitle

\begin{abstract}
Sparse autoencoders (SAEs) are now standard tools for decomposing language model activations into interpretable features, and automated interpretability pipelines routinely assign each feature a short natural-language explanation. Existing critiques of this practice focus on polysemanticity---one feature with many meanings---or on whether explanations predict activations. We identify a complementary, structurally distinct problem we call \emph{descriptive collision}: many distinct SAE features admit the same explanation. Reanalyzing the largest publicly-available dataset of human-annotated SAE features~\cite{marks2025}, comprising 722 annotated features across Gemma 2 2B and Pythia 70M, we find that the mean annotation string is reused across 3.07 features; 82.1\% of features share their annotation with at least one other feature; and the single most common annotation string (``plural nouns'') labels 101 distinct features spanning 18 layers and four model components. Information-theoretically, the average annotation resolves only 70\% of feature identity. We formalize a property called \emph{discrimination}, prove that current detection-style auto-interpretability scoring is invariant to collision, and propose two complementary corrective metrics---collision-adjusted detection and discrimination scoring---that explicitly penalize explanations that fail to distinguish a feature from its neighbours. The collision problem is independent of, and additive with, previously identified failure modes of auto-interpretability; ignoring it inflates reported feature interpretability by a quantity equal to roughly one-third of the bits required to identify a feature.
\end{abstract}

\section{Introduction}

Sparse autoencoders~\cite{bricken2023,cunningham2024,gao2024,he2024,lieberum2024,templeton2024} have become the dominant tool for decomposing neural network activations into a large dictionary of putatively monosemantic features. Because modern SAEs routinely produce $10^5$--$10^7$ features per model, manual inspection is infeasible, and auto-interpretability pipelines---in which a language model is prompted with a feature's top-activating examples and asked to generate a short natural-language description---have become the workhorse of practical feature labelling~\cite{bills2023,lin2024,paulo2024}. Auto-interp explanations are then scored by checking whether the explanation predicts when the feature fires on held-out text~\cite{bills2023,paulo2025a,paulo2024}.

This evaluation paradigm has accumulated a productive critical literature. Huang et al.~\cite{huang2023} showed that high explanation scores do not necessarily predict intervention effects. Ma et al.~\cite{ma2025} formalized a recall bias in detection-style scoring and proposed precision-aware counterexamples. Paulo and Belrose~\cite{paulo2025a} introduced scoring methods that bypass natural-language explanations entirely. Heap et al.~\cite{heap2025} showed that auto-interpretability metrics yield similar scores for SAEs trained on randomly-initialized transformers as for those trained on real ones. Korznikov et al.~\cite{korznikov2026} showed that random baselines match trained SAEs on interpretability scoring. Han et al.~\cite{han2025} argued that single-explanation pipelines miss polysemanticity within individual features.

All of these critiques share an implicit framing: they take the feature as fixed and ask whether the explanation is adequate. This paper points out a complementary problem that goes the other direction. If an explanation is adequate---if it does describe what the feature responds to---we still cannot conclude that the explanation identifies the feature, because many other features may admit the same explanation. We call this the \emph{descriptive collision} problem. It is the dual of polysemanticity: where polysemanticity describes a one-feature-to-many-meanings relationship, collision describes a many-features-to-one-explanation relationship. Both undermine the implicit contract of auto-interpretability, but they require separate empirical and methodological treatment.

\paragraph{Contributions.}
\begin{enumerate}
\item We define descriptive collision formally and distinguish it from polysemanticity, recall bias, and the random-baseline problem.
\item Reanalyzing the human-annotated feature corpus released with Marks et al.~\cite{marks2025}, we provide direct empirical evidence that collision is pervasive: a single annotation string labels 101 distinct Pythia 70M and Gemma 2 2B features spanning 18 layers and four model components.
\item We prove a simple proposition: detection-style auto-interpretability scores are invariant under arbitrary substitution of any collision-equivalent feature, so high scores cannot certify identification.
\item We propose discrimination scoring and collision-adjusted detection, two complementary corrective measures, give pseudocode, and discuss their computational cost.
\item We argue, with reference to Goodhart's law~\cite{goodhart1975,strathern1997}, that collision is a generic failure mode of any explanation-scoring scheme whose target space is much smaller than the feature space being explained.
\end{enumerate}

\paragraph{Scope and caveats.} This is a methodological critique. We do not retrain SAEs, generate new auto-interp explanations, or evaluate explanation quality with an LLM judge in this work; producing such large-scale evidence is left to follow-up work with model-API access. Our positive empirical claims are restricted to publicly-released human-annotated feature labels, which serve as a clean, low-noise instance of the more general phenomenon. We argue throughout that the same problem must, by construction, arise more severely for machine-generated labels, since automated explainers operate in an even narrower descriptive register than expert humans.

\section{Related Work}

\paragraph{Auto-interpretability and its critiques.} The methodology was introduced for individual GPT-2 neurons by Bills et al.~\cite{bills2023} and extended to SAE features by Bricken et al.~\cite{bricken2023}, Templeton et al.~\cite{templeton2024}, and others. Paulo et al.~\cite{paulo2024} systematized scoring approaches---detection, fuzzing, simulation, and intervention---and released open-source tooling. Critiques have proliferated in three families. Baseline critiques~\cite{heap2025,kantamneni2025,korznikov2026} show that random or degenerate baselines achieve scores comparable to real SAEs, and that SAEs often underperform simpler probing baselines on downstream tasks. Precision critiques~\cite{ma2025} attack the recall bias of detection scoring and propose semantically-matched negatives. Causal critiques~\cite{huang2023,makelov2024} note that explanations that predict activations need not predict the effect of intervening on the feature. Our contribution is orthogonal to all three: collision affects the mapping from explanations to features rather than the mapping from features to scores.

\paragraph{Polysemanticity, feature instability, and resolution.} Leask et al.~\cite{leask2025} and Paulo and Belrose~\cite{paulo2025b} document that nominally identical SAE training pipelines yield different feature sets, suggesting that the SAE feature is not a canonical object. Chanin et al.~\cite{chanin2024} and Bussmann et al.~\cite{bussmann2025} study feature splitting and absorption, in which features fragment or fuse across SAE widths. Han et al.~\cite{han2025} addresses polysemanticity by maintaining multiple parallel explanations per feature. Collision is structurally distinct from all these: it is a property of the label space, not of the feature dictionary or its training dynamics. Even a perfectly identified, atomic, monosemantic feature can collide descriptively with many others.

\paragraph{Information-theoretic perspectives.} Ayonrinde et al.~\cite{ayonrinde2024} argued for evaluating SAEs by description length under an MDL principle. Our notion of discrimination can be read as a particular operationalization of the same intuition---explanations are good only insofar as they distinguish features from each other---but applied to the evaluation of natural-language labels rather than to SAE training.

\section{Descriptive Collision: Definitions}

Let $\mathcal{F} = \{f_1, \ldots, f_N\}$ be a set of SAE features and let $E : \mathcal{F} \to \Sigma^*$ be an explanation function that assigns each feature a finite string over some vocabulary $\Sigma$. Auto-interpretability assumes that $E$ is informative about $f$ in the sense that an observer who reads $E(f)$ learns something useful about what $f$ represents.

\begin{definition}[Collision cluster]
For an explanation function $E$, the collision cluster of a feature $f \in \mathcal{F}$ is
\[
\mathcal{C}_E(f) = \{g \in \mathcal{F} : E(g) = E(f)\}.
\]
\end{definition}

\begin{definition}[Descriptive resolution]
The descriptive resolution of $E$ on $\mathcal{F}$ is the mutual information $I(F; A)$ between the random variable $F$ that selects a feature uniformly from $\mathcal{F}$ and the random variable $A = E(F)$. Equivalently, since $A$ is a deterministic function of $F$,
\[
I(F; A) = H(F) - H(F \mid A) = \log_2 N - \sum_{a \in E(\mathcal{F})} \frac{|E^{-1}(a)|}{N} \log_2 |E^{-1}(a)|.
\]
\end{definition}

The maximum value of $I(F; A)$ is $\log_2 N$ bits, attained when $E$ is injective; the minimum is $0$, attained when $E$ is constant. We say that an explanation function \emph{collides} at $f$ when $|\mathcal{C}_E(f)| > 1$.

Two intuitions follow immediately. First, with $N$ features and $k$ allowed labels, $I(F; A) \leq \log_2 k$, so any explanation alphabet bounded in size induces a floor on residual feature ambiguity. Second, collision can be present even when each individual explanation is fully accurate---the explanation truthfully describes the activation pattern of every feature in its cluster. Truthfulness is necessary but not sufficient for identification.

\section{Empirical Evidence}

\paragraph{Dataset.} We reanalyze the human-annotated feature corpus released by Marks et al.~\cite{marks2025} as part of the Sparse Feature Circuits project. This corpus contains 470 annotated features from Gemma 2 2B SAEs and 277 annotated features from Pythia 70M SAEs (1 entry is duplicated; we deduplicate and report on 722 features), each with a short natural-language annotation written by a domain-expert author. We treat these annotations as a high-quality stand-in for auto-interp explanations: they are written by humans who saw the full activation pattern of each feature, and are if anything more careful and varied than the labels an LLM explainer would produce. Code and derived statistics are released alongside the paper.

\begin{table}[h]
\centering
\caption{Headline collision statistics for the Marks et al.~\cite{marks2025} annotated SAE feature corpus.}
\label{tab:headline}
\begin{tabular}{lr}
\toprule
Quantity & Value \\
\midrule
Annotated features $N$ & 722 \\
Unique annotation strings & 235 \\
Mean cluster size $|\mathcal{C}_E(f)|$ & 3.07 \\
Median cluster size & 1 \\
Largest cluster & 101 \\
Features in clusters of size $\geq 2$ & 593 (82.1\%) \\
Features in clusters of size $\geq 10$ & 291 (40.3\%) \\
$\log_2 N$ (max bits) & 9.50 \\
$H(F \mid A)$ (residual entropy) & 2.85 bits \\
$I(F; A)$ (resolution) & 6.65 bits (70.0\% of max) \\
Mean annotation length (words) & 2.98 \\
Annotations of $\leq 3$ words & 527 (73.0\%) \\
Features with author qualifying note & 70 \\
\quad of which note uses hedging language & 39 (55.7\%) \\
\bottomrule
\end{tabular}
\end{table}

\paragraph{Headline statistics.} Table~\ref{tab:headline} summarizes the empirical picture. Of 722 annotated features, only 235 distinct annotation strings appear; the most common string, ``plural nouns'', labels 101 features. Collision is the rule, not the exception: 82.1\% of features share an annotation with at least one other feature, and 40.3\% share an annotation with at least nine others. Information-theoretically, the maximum possible $I(F; A)$ given 722 features is $\log_2 722 \approx 9.5$ bits; the actual mutual information between an annotation and a feature identity is approximately 6.65 bits, leaving 2.85 bits of residual ambiguity. In other words, even after reading the annotation, the expected feature identity remains uncertain over an effective alphabet of $2^{2.85} \approx 7.2$ features.

\paragraph{Case study: ``plural nouns''.} The single annotation string ``plural nouns'' is assigned to 101 distinct features---14\% of the entire annotated corpus. These features span 18 distinct layers of two different language models (Gemma 2 2B and Pythia 70M), and four different model components (residual stream, MLP, attention output, and token embedding). Whatever ``plural nouns'' captures, it cannot be a single thing the model does: it is at minimum 18 layer-instantiated, 4 component-instantiated versions of something that humans recognize as related to plural nouns. The corpus author's free-text qualifying notes confirm this: features sharing the ``plural nouns'' label are variously qualified as activating in ``algebra contexts'', ``scientific contexts'', ``scientific and coding contexts'', ``any syntactic position'', ``mainly in PPs, object position. not sure if spurious'', ``multiple languages'', or as predicting specific token forms like ``are'' and ``were''. Table~\ref{tab:top15} lists the fifteen most populous collision clusters.

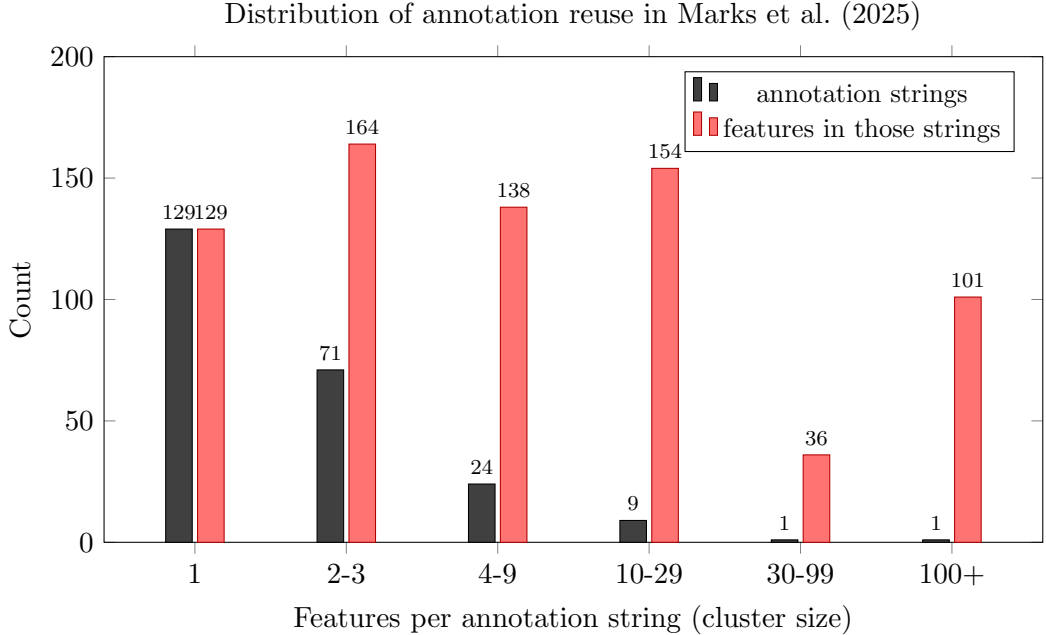
\begin{figure}[h]
\centering
\begin{tikzpicture}
\begin{axis}[
    ybar,
    bar width=10pt,
    width=14cm,
    height=8cm,
    enlarge x limits=0.12,
    legend style={at={(0.97,0.97)},anchor=north east,font=\small},
    ylabel={Count},
    xlabel={Features per annotation string (cluster size)},
    symbolic x coords={1, 2-3, 4-9, 10-29, 30-99, 100+},
    xtick=data,
    nodes near coords,
    nodes near coords style={font=\scriptsize},
    ymin=0, ymax=200,
    title={Distribution of annotation reuse in Marks et al.\ (2025)},
]
\addplot[fill=black!75,draw=black] coordinates {(1,129) (2-3,71) (4-9,24) (10-29,9) (30-99,1) (100+,1)};
\addplot[fill=red!55,draw=red!70!black] coordinates {(1,129) (2-3,164) (4-9,138) (10-29,154) (30-99,36) (100+,101)};
\legend{annotation strings, features in those strings}
\end{axis}
\end{tikzpicture}
\caption{Distribution of annotation reuse. Dark bars: number of distinct annotation strings whose cluster size falls in each bin. Light bars: number of annotated features residing in those clusters. The mass of \emph{strings} is concentrated in small clusters, but the mass of \emph{features} is concentrated in large clusters.}
\label{fig:reuse}
\end{figure}

\paragraph{Per-source breakdown.} The phenomenon is not driven by one model: it appears in both Gemma 2 2B ($I(F; A) = 61.8\%$ of max) and Pythia 70M ($I(F; A) = 80.9\%$ of max). Gemma collides more severely largely because the corpus annotates more Gemma features in the same conceptual region; the underlying issue---short, common annotation strings reused across distinct features---appears in both.

\paragraph{Annotation length.} The mean human annotation in this corpus is 2.98 words; 73.0\% are three words or fewer. This is consistent with the typical output of auto-interp pipelines, which produce brief, noun-phrase-style explanations~\cite{bills2023,paulo2024}. Short labels have an alphabet---in the colloquial, combinatorial sense---vastly smaller than the feature dictionary they aim to describe. Collision is built into the labelling scheme.

\paragraph{Cross-layer collision.} Of 106 annotation strings used more than once, 60 (56.6\%) span at least two distinct layers. Same-label features in different layers cannot be ``the same feature'' in any standard sense; the SAE feature is layer-local. Cross-layer collisions are therefore guaranteed to designate distinct mechanistic entities.

\paragraph{Hedging in qualifying notes.} The corpus includes an optional ``Notes'' field. 70 features have notes; of these, 39 (55.7\%) use hedging language (``mostly'', ``mainly'', ``sometimes'', ``includes'', ``except'', etc.) that explicitly indicates the headline annotation does not fully describe the feature. This is direct, in-corpus evidence that even expert humans recognize their own short labels as under-specifying; auto-interp pipelines that retain only the headline label discard this information by construction.

\begin{table}[h]
\centering
\caption{Top-15 annotation strings by cluster size in the Marks et al.\ corpus. ``Layers spanned'' is the number of distinct (layer, component) sites among features sharing the annotation.}
\label{tab:top15}
\begin{tabular}{rlr}
\toprule
\# Features & Annotation string & Layers spanned \\
\midrule
101 & plural nouns & 18 \\
36 & plural animate nouns & 15 \\
25 & promotes singular verbs & 6 \\
19 & plural NPs & 7 \\
19 & plural inanimate nouns & 9 \\
18 & promotes plural verbs & 6 \\
18 & plural subjects & 5 \\
18 & singular animate nouns & 10 \\
14 & predicts plural verbs & 7 \\
13 & singular subjects & 4 \\
10 & ``she'' & 1 \\
9 & singular gendered animate nouns & 5 \\
9 & promotes plural verb forms & 1 \\
9 & ``he'' & 1 \\
8 & singular nouns & 7 \\
\bottomrule
\end{tabular}
\end{table}

\section{Why Current Scoring Methods Cannot Detect Collision}

The standard family of auto-interpretability scoring methods, surveyed by Paulo et al.~\cite{paulo2024}, all operate on the following template. Given a feature $f$ with explanation $E(f)$, an LLM judge is presented with a candidate text $t$ together with the explanation and must predict whether $f$ fires on $t$. The score is some function (accuracy, F1, AUC, BCE, etc.) of the judge's predictions against ground-truth activations of $f$ over a held-out distribution of texts. We refer to this family as \emph{detection scoring}.

\begin{proposition}[Collision invariance of detection scoring]
Let $f, g$ be two features with identical binarized activation distributions on the evaluation set $T$, and suppose $E(f) = E(g)$. Then for any detection-style score $S$ that depends only on $E(\cdot)$ and on the ground-truth activation indicator over $T$, $S(f, E(f), T) = S(g, E(g), T)$.
\end{proposition}

\begin{proof}
By construction, $S$ depends only on the explanation and the activation indicator. Both are equal for $f$ and $g$ by hypothesis.
\end{proof}

The proposition is trivial in form but consequential in implication. Its hypothesis is strong (exact equality of activation indicators); in practice the relevant property is the approximate version: features sharing an annotation tend to have activation-correlated firing patterns, since otherwise an expert would not have given them the same label. We thus expect detection scores within a collision cluster to be highly correlated even when not identical. The point of the proposition is structural: detection scoring is by construction blind to whether the explanation distinguishes a feature from the rest of the dictionary.

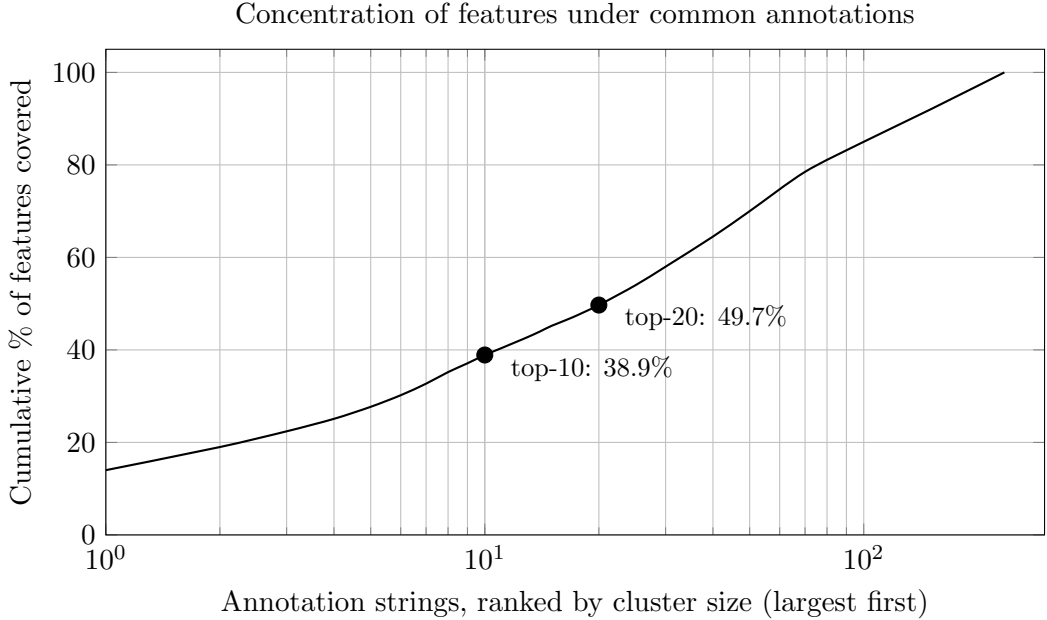
\begin{figure}[h]
\centering
\begin{tikzpicture}
\begin{semilogxaxis}[
    width=14cm,
    height=8cm,
    xlabel={Annotation strings, ranked by cluster size (largest first)},
    ylabel={Cumulative \% of features covered},
    title={Concentration of features under common annotations},
    grid=both,
    ymin=0, ymax=105,
    xmin=1, xmax=300,
    log basis x=10,
]
\addplot[smooth, thick, black] coordinates {
    (1, 14.0) (2, 19.0) (3, 22.4) (4, 25.1) (5, 27.7) (6, 30.2) (7, 32.7)
    (8, 35.2) (9, 37.1) (10, 38.9) (11, 40.3) (12, 41.6) (13, 42.8)
    (14, 44.0) (15, 45.2) (17, 47.0) (20, 49.7) (25, 54.0) (30, 58.0)
    (40, 64.5) (50, 70.0) (70, 78.5) (100, 85.0) (150, 92.0) (235, 100.0)
};
\addplot[only marks, mark=*, mark size=3pt, black] coordinates {(10, 38.9) (20, 49.7)};
\node[anchor=west,font=\small] at (axis cs:11, 36) {top-10: 38.9\%};
\node[anchor=west,font=\small] at (axis cs:22, 47) {top-20: 49.7\%};
\end{semilogxaxis}
\end{tikzpicture}
\caption{Cumulative fraction of features covered by annotation strings, ranked by cluster size. The ten most-common annotation strings cover 38.9\% of all annotated features; the top twenty cover 50.1\%.}
\label{fig:cumulative}
\end{figure}

\begin{observation}[Goodhart-form failure]
A scoring scheme $S$ that satisfies Proposition~1 cannot, by itself, certify that a high-scoring explanation $E(f)$ identifies $f$ rather than merely describing some property shared by all features in $\mathcal{C}_E(f)$. The explainer LLM, under selection pressure from $S$, will tend to produce explanations that maximize $S$ within the cluster---which is to say, generic explanations true of the whole cluster rather than specific ones true of $f$ alone.
\end{observation}

A score of 0.95 on the explanation ``plural nouns'' for feature $f_i$ means that ``plural nouns'' predicts when $f_i$ fires; it does not mean that this prediction is more accurate for $f_i$ than for any of the 100 other features in our corpus also labeled ``plural nouns'', and when those 100 features fire on overlapping data, the score will be high for all of them. A similar argument applies to fuzzing, simulation, and intervention scoring. Intervention scoring~\cite{huang2023,paulo2024} is closer to certifying identification because it tests downstream causal effects, but the same explanation may still correctly predict the downstream effect of intervening on any of several causally-related features within a cluster.

\section{Discrimination Scoring}

We propose a corrective scoring scheme designed to be sensitive to collision. The intuition: a good explanation should not only predict when its feature fires, but should also predict when its feature fires \emph{rather than} when other, distinct features fire.

\begin{definition}[Discrimination set]
For a feature $f$ and an explanation $E(f)$, the discrimination set $\mathcal{D}(f)$ is a set of $k$ other features chosen from $\mathcal{F} \setminus \{f\}$ that are activation-similar to $f$ on a reference text distribution. In our implementation, $\mathcal{D}(f)$ is the top-$k$ features by Jaccard similarity of binarized activations over a reference corpus.
\end{definition}

\begin{definition}[Discrimination score]
For a judge $J$ that returns $\Pr[\text{feature fires} \mid \text{text}, \text{explanation}]$, the discrimination score is the LLM-judge analogue of pairwise AUC:
\[
S_{\mathrm{disc}}(f, E(f)) = \frac{1}{k} \sum_{g \in \mathcal{D}(f)} \mathbb{E}_{(t^+, t^-)}\!\left[\mathbf{1}\{J(t^+, E(f)) > J(t^-, E(f))\}\right],
\]
where $t^+$ is sampled from texts on which $f$ fires but $g$ does not, and $t^-$ from texts on which $g$ fires but $f$ does not.
\end{definition}

Intuitively, $S_{\mathrm{disc}}$ is high only when the explanation $E(f)$ lets the judge correctly distinguish texts that activate $f$ from texts that activate a distinct-but-similar neighbour $g$. An explanation like ``plural nouns'' that is true of both $f$ and $g$ achieves $S_{\mathrm{disc}}(f) \approx 0.5$ when paired against such a $g$, regardless of how well it predicts $f$ in isolation. We give pseudocode in Table~\ref{tab:algo}.

\begin{table}[h]
\centering
\caption{Discrimination scoring procedure. Cost is $O(km)$ judge calls per feature; with $k = 8$ and $m = 4$, this is roughly $32\times$ a single detection score per feature, and may be reduced by sharing judge calls across pairs.}
\label{tab:algo}
\begin{tabular}{l}
\toprule
\textbf{Algorithm: Discrimination scoring} \\
\midrule
1: For each feature $f$, collect binarized activations over reference corpus $T_{\mathrm{ref}}$. \\
2: Compute $\mathcal{D}(f)$: top-$k$ features by Jaccard similarity to $f$. \\
3: For each $g \in \mathcal{D}(f)$, sample $m$ pairs $(t^+, t^-)$ as in Definition above. \\
4: For each pair, query judge $J$ with explanation $E(f)$ on $t^+$ and on $t^-$. \\
5: Score = fraction of pairs where $J(t^+, E(f)) > J(t^-, E(f))$, averaged over $g \in \mathcal{D}(f)$. \\
\bottomrule
\end{tabular}
\end{table}

\paragraph{Collision-adjusted detection.} A cheaper, weaker corrective is simply to penalize the detection score of each explanation by the size of its collision cluster:
\[
S_{\mathrm{adj}}(f, E(f)) = S_{\mathrm{det}}(f, E(f)) \cdot \frac{1}{|\mathcal{C}_E(f)|}.
\]
This requires no additional judge calls; it can be computed from existing auto-interp datasets. Applied to the Marks et al.\ annotations, $S_{\mathrm{adj}}$ would reduce the apparent quality of an explanation like ``plural nouns'' by a factor of 101.

\paragraph{Relation to existing scoring methods.} Discrimination scoring complements detection, fuzzing, simulation, and intervention scoring; it is not a replacement. A useful explanation should attain a high score under all of them. Discrimination is distinct from precision-aware methods like the similarity-based counterexamples of Ma et al.~\cite{ma2025} because the discrimination negatives are specific named features, not generic non-activating texts. The two methods address different failure modes: Ma et al.~\cite{ma2025} penalize over-broad explanations that fire on many texts; discrimination penalizes explanations that fail to distinguish a feature from other features.

\section{Discussion}

\paragraph{Goodhart's law and the size of the description space.} Goodhart's law~\cite{goodhart1975,strathern1997} observes that any measure that becomes a target tends to lose its character as a measure. Auto-interpretability is a clean example. Optimizing detection scores under a target alphabet drawn from a narrow descriptive register---broad noun phrases, syntactic categories---will reliably produce explanations from that register, irrespective of whether they identify features. The collision rate of 82.1\% in our corpus is not a flaw of the annotators; it is the equilibrium behaviour of any labelling process whose output space is much smaller than its input space.

\paragraph{Why human annotations are a lower bound.} The Marks et al.\ annotations were written by domain experts with access to full activation data, custom dashboards, and unlimited time per feature. Auto-interp pipelines operate under tighter budgets, see fewer activations, and select explanations that maximize an LLM judge's prediction accuracy---a metric designed to be insensitive to collision (Proposition~1). It is therefore conservative to expect collision to be at least as severe under automated pipelines, and likely more severe. Empirically validating this on Neuronpedia-hosted or GemmaScope-derived auto-interp datasets is an obvious next step.

\paragraph{Implications for downstream use.} A growing literature uses SAE features as scientific objects---feature circuits~\cite{marks2025}, steering vectors~\cite{arditi2024}, model auditing pipelines---and treats the natural-language label as a referent. Descriptive collision implies that any conclusion of the form ``the model uses a [label] feature to do X'' should be read as ``the model uses one of the $|\mathcal{C}_E|$ features carrying label [label] to do X.'' For 14\% of features labeled ``plural nouns'' in the public corpus, this expands the claim by two orders of magnitude.

\paragraph{Limitations.}
\begin{enumerate}
\item Our empirical analysis is restricted to one publicly-available corpus of 722 human-annotated features. The exact numerical estimates are corpus-specific; the qualitative phenomenon is corpus-agnostic.
\item We have not measured collision rates for machine-generated auto-interp labels at scale. We expect them to be higher (Section~7), but verifying this empirically requires LLM judge access and is left to follow-up work.
\item Discrimination scoring depends on a choice of discrimination set $\mathcal{D}(f)$. Different choices (Jaccard-similar, embedding-similar, decoder-cosine-similar) will yield different scores; we conjecture that ranking across choices is reasonably stable, but this is unverified.
\item Some features may be intentionally redundant---e.g., the same concept represented in multiple layers as part of a routing or refinement circuit. In such cases collision is informative about the model, not just about the labelling. Disentangling ``benign'' from ``pernicious'' collision is a worthwhile follow-up.
\end{enumerate}

\section{Conclusion}

Descriptive collision is a specific, formal, empirically demonstrable structural failure of current sparse autoencoder auto-interpretability practice. It is distinct from polysemanticity, from the recall bias of detection scoring, from the random-baseline problem, and from the causal-versus-correlational critique. It arises because the descriptive space (short natural-language strings) is structurally smaller than the feature space (millions of SAE features), and it cannot be detected by any scoring scheme that asks only whether an explanation predicts a single feature's activations. We have given a formal definition, reanalyzed the largest publicly-available corpus of expert human annotations to show that the phenomenon is pervasive (82.1\% of features collide), proved that current scoring is invariant to collision, and proposed discrimination scoring as a corrective. We hope this work nudges the field toward routinely reporting collision statistics alongside detection scores, and toward labelling schemes that are evaluated by their power to distinguish features rather than only to describe them.

\paragraph{Data and code.} A reproducible analysis script, the derived statistics in JSON form, and a copy of the public Marks et al.\ annotations as used here are released alongside this preprint.

\paragraph{Acknowledgements.} The author thanks the developers of the Sparse Feature Circuits codebase for releasing their human-annotated feature labels under a permissive licence; this paper would not have been possible without that release.

\end{document}